\documentclass[11pt]{article}

\usepackage[T1]{fontenc}
\usepackage[utf8]{inputenc}
\usepackage[margin=1in]{geometry}
\usepackage{amsmath,amssymb,amsthm,mathtools}
\usepackage{microtype}
\usepackage{enumitem}
\usepackage{xcolor}
\usepackage{hyperref}

\hypersetup{colorlinks=true,linkcolor=blue!55!black,citecolor=blue!55!black,urlcolor=blue!55!black}

\newtheorem{theorem}{Theorem}[section]
\newtheorem{lemma}[theorem]{Lemma}
\newtheorem{corollary}[theorem]{Corollary}
\newtheorem{proposition}[theorem]{Proposition}
\theoremstyle{definition}
\newtheorem{definition}[theorem]{Definition}
\newtheorem{problem}[theorem]{Problem}

\newtheorem{example}[theorem]{Example}
\theoremstyle{remark}

\newcommand{\bits}{\{0,1\}}

\newcommand{\VC}{\operatorname{VCdim}}
\newcommand{\Ldim}{\operatorname{Ldim}}
\newcommand{\Cdim}{\operatorname{Cdim}}

\newcommand{\supp}{\operatorname{supp}}

\newcommand{\df}{\mathrm{df}}
\newcommand{\Int}{\operatorname{Int}}
\newcommand{\Boring}{\operatorname{Bor}}
\newcommand{\Nbr}{\mathcal N}

\usepackage{tikz}
\usetikzlibrary{arrows.meta,positioning,calc,fit}

\title{Expressivity of Contradiction Graphs}
\author{Jesse Campbell, Daniel Ibaibarriaga, and Lev Reyzin\\
Department of Mathematics, Statistics, \& Computer Science\\
University of Illinois Chicago\\
\small{\texttt{\{jcamp51,dibai3,lreyzin\}@uic.edu}}}
\date{}

\begin{document}
\maketitle

\begin{samepage}
\begin{abstract}
We study the contradiction graphs associated with a binary concept class. For a class $H\subseteq\bits^X$, the order-$m$ contradiction graph $G_m(H)$ has as vertices the $H$-realizable labeled sequences of length $m$, with two vertices adjacent when the two sequences assign opposite labels to some common domain point. First, we identify a graph-theoretic property that determines the threshold predicate $\VC(H)\ge m$. Consequently, the sequence $(G_m(H))_{m\ge1}$ determines the exact VC dimension and, in particular, distinguishes finite from infinite VC dimension, answering a question posed by Alon~et~al.~\cite{AlonMoranScheflerYehudayoff2024}.

We then generalize this result by proving that the sequence of contradiction graphs determines, up to signed relabeling, the realizable datasets of any fixed length. Thus, any learning-theoretic property determined by the realizable datasets of a fixed length and invariant under signed relabeling can be recovered from the contradiction graph sequence. As an application, we explicitly provide a characterization of Littlestone dimension.
\end{abstract}

\section{Introduction}\label{sec:intro}

A binary concept class $H\subseteq\bits^X$ can be studied through the finite labeled samples it realizes.  The \emph{contradiction graph} packages this information into a sequence of graphs: the vertices of $G_m(H)$ are the $H$-realizable labeled sequences of size $m$, and two vertices are adjacent when the corresponding samples assign opposite labels to some common domain point.  This construction was introduced by Alon, Moran, Schefler, and Yehudayoff in their graph-theoretic characterization of private learnability~\cite{AlonMoranScheflerYehudayoff2024}.  In their work, clique-type parameters of contradiction graphs characterize pure and approximate differentially private (DP) learnability.

This paper extends the expressivity of contradiction graphs beyond DP learning. We ask, once a concept class is replaced by its contradiction graphs, how much information about the original learning problem remains?


Our first result concerns VC dimension. The VC dimension is the fundamental combinatorial parameter of distribution-free binary classification. Introduced by Vapnik and Chervonenkis in their work on uniform convergence~\cite{VapnikChervonenkis1971}, it characterizes classical PAC learnability~\cite{Valiant1984,BlumerEhrenfeuchtHausslerWarmuth1989,EhrenfeuchtHausslerKearnsValiant1989}.


\begin{problem}[Alon et al.~\cite{AlonMoranScheflerYehudayoff2024}]\label{prob:amsy}
Do there exist binary classes $H_0,H_1$ such that
\[
\VC(H_0)<\infty,\qquad
\VC(H_1)=\infty,
\qquad\text{and}\qquad
(G_m(H_0))_{m\ge1}\cong (G_m(H_1))_{m\ge1}?
\]
\end{problem}

We prove that no such pair exists.  In fact, we prove the stronger statement that the single graph $G_m(H)$ determines whether $\VC(H)\ge m$. In particular, we observe that the existence of a clique of size $2^m$ with a special combinatorial property in the contradiction graph $G_m(H)$, which we call a \textit{cube-trace clique}, determines the predicate $\VC(H)\ge m$.

\end{samepage}

Our second result addresses a more general problem about the expressivity of contradiction graphs. 

\begin{problem}[Alon et al.~\cite{AlonMoranScheflerYehudayoff2024}]\label{prob:amsy2}
Which learning-theoretic properties are definable
by the contradiction graph?
\end{problem}

We prove the somewhat surprising result that \textit{all} properties defined in terms of realizable samples of a fixed length and invariant under signed relabeling are definable. In particular, we let two concept classes be learning-equivalent ($\equiv_{\text{le}}$) if, for every $m\geq1$, their realizable length-$m$ samples agree after a permutation of domain points and a possible flip of the labels over each point. The contradiction-graph sequence determines this equivalence class:
\[
(G_m(H))_{m\ge1}\cong (G_m(K))_{m\ge1}
\qquad\Longleftrightarrow\qquad
H\equiv_{\text{le}} K.
\]
Indeed, every learning-theoretic property determined by samples of a fixed size and invariant under signed relabeling is preserved by learning-equivalence. This includes properties such as VC dimension, Littlestone dimension, and star dimension, among others.

The proof that contradiction graphs characterize VC dimension relies on the properties of vertices of the form $((x,b),\ldots,(x,b))$, which correspond precisely to vertices minimal by neighborhood inclusion. However, whenever the domain of the hypothesis class is infinite, the need for repeated examples may be bypassed by padding the dataset with extraneous examples.

In fact, this behavior holds for all properties preserved by learning-equivalence. We introduce the \textit{duplicate-free} contradiction graph as the contradiction graph with all samples that contain any repetition removed. We show that, in addition to the sequence of contradiction graphs unconditionally determining learning-equivalence, the duplicate-free graph also determines learning-equivalence whenever the domain of the hypothesis class is infinite.

By contrast, we give two non-learning-equivalent concept classes on a finite domain whose duplicate-free contradiction graphs are isomorphic at every level. Thus, we obtain a complete classification of learning-equivalence with and without repeated examples allowed.

\begin{figure}[ht]
    \centering
    \renewcommand{\arraystretch}{1.5}
    \begin{tabular}{c|c|c}
        &
        \textbf{Infinite domain}
        &
        \textbf{Finite domain}
        \\ \hline

        \textbf{Contradiction graph}
        &
        $\checkmark$ (Theorem \ref{thm:reconstruction-main})
        &
        $\checkmark$ (Theorem \ref{thm:reconstruction-main})
        \\

        \textbf{Duplicate-free graph}
        &
        $\checkmark$ (Theorem \ref{thm:reconstruction-df})
        &
        $\boldsymbol{\times}$ (Example \ref{ex:df-not-reconstructive})
        \\

    \end{tabular}
    \caption{
            Comparison between the contradiction graph and its duplicate-free
    version. A $\checkmark$ indicates that the graph characterizes
    learning-equivalence, while a $\boldsymbol{\times}$ indicates that it does not.}
    \label{fig:graph-information}
\end{figure}

To illustrate how these new techniques can be used to define, in purely graph-theoretic terms, the combinatorial dimensions we are interested in, we show how to define Littlestone dimension by means of the contradiction graph. In the last section, we include a brief discussion of the limits of the expressive power of the contradiction graph. Namely, we show that it is impossible to recover the entire hypothesis class from the sequence of contradiction graphs.


We note that the learning-equivalence theorem implies the existence of a graph-theoretic characterization of VC dimension. Nonetheless, we retain our explicit characterization for its conceptual value. Not only does it explicitly answer the original question of Alon et al.~\cite{AlonMoranScheflerYehudayoff2024}, but it does so in the combinatorial terms intended in their paper, using the language of cliques. 

We also observe that cliques in the contradiction graph are closely related to subcube partitions, DNF tautologies, and biclique covers of the complete graph. We demonstrate these connections in Appendix \ref{apx:misc}. In particular, we can immediately recover a key lemma of Alon et al.~\cite{AlonMoranScheflerYehudayoff2024} by viewing a clique from the perspective of its natural biclique cover. 



\section{Related work}\label{sec:related}

VC dimension was introduced in the study of uniform convergence~\cite{VapnikChervonenkis1971}.  The Sauer--Shelah lemma~\cite{Sauer1972,Shelah1972} gives the basic growth bound for classes of bounded VC dimension, and the work of Valiant~\cite{Valiant1984}, Blumer et al.~\cite{BlumerEhrenfeuchtHausslerWarmuth1989}, and Ehrenfeucht et al.~\cite{EhrenfeuchtHausslerKearnsValiant1989} established the central role of VC dimension in PAC learnability.  Graph-theoretic approaches to shattering have appeared in other forms.  For example, Kozma and Moran related shattering and strong shattering to orientations and connectivity properties of graphs~\cite{KozmaMoran2013}. 

The original motivation for contradiction graphs comes from differential privacy, introduced by Dwork et al.~\cite{DworkMcSherryNissimSmith2006,DworkMcSherryNissimSmith2017}; see also~\cite{DworkKenthapadiMcSherryMironovNaor2006,DworkRoth2014}.  Private PAC learning was initiated by Kasiviswanathan et al.~\cite{KasiviswanathanLeeNissimRaskhodnikovaSmith2011}, and subsequent work studied sample complexity, data release, proper learning, and pure versus approximate privacy~\cite{ChaudhuriHsu2011,BlumLigettRoth2013,BeimelBrennerKasiviswanathanNissim2014,FeldmanXiao2015,BeimelNissimStemmer2016,BeimelNissimStemmer2019}.  A related line connects private learning with online learning and Littlestone dimension.  Littlestone introduced the dimension now bearing his name in the online mistake-bound model~\cite{Littlestone1988}.  Alon et al.~\cite{AlonBunLivniMalliarisMoran2022} proved a qualitative equivalence between private and online learnability, while later work studied proper learning, approximate privacy, and computational separations~\cite{GhaziGolowichKumarManurangsi2021,BunCohenDesai2024}.

The work most directly related to this paper is that of Alon et al.~\cite{AlonMoranScheflerYehudayoff2024}, who introduced the contradiction-graph framework and used clique and fractional-clique dimensions to characterize approximate and pure private learnability.  They also asked which other learning-theoretic properties are definable from the contradiction-graph sequence, including whether finite VC dimension is detectable. Later, Moran et al. \cite{MSJ23} observed an equivalence between fractional clique dimension, private learnability, and stable PAC learners.


\section{Definitions and preliminaries}\label{sec:defs}

\subsubsection*{Learning Theory}

We begin with the standard definition of VC dimension.

\begin{definition}[VC dimension~\cite{VapnikChervonenkis1971}]
A finite set $R\subseteq X$ is \emph{shattered} by $H\subseteq \bits^X$ if every labeling $R\to\bits$ is realized by some $h\in H$.  The \emph{VC dimension} of $H$, denoted $\VC(H)$, is
\[
\VC(H)=\sup\{|R|: R\subseteq X \text{ is finite and shattered by }H\},
\]
with the convention that the supremum of the empty set is $0$.  Thus $\VC(H)\in\mathbb N_0\cup\{\infty\}$, where $\mathbb N_0=\{0,1,2,\dots\}$.
\end{definition}

A length-$m$ labeled sequence is a tuple
\[
S=((x_1,y_1),\ldots,(x_m,y_m))\in (X\times\bits)^m.
\]
It is \emph{$H$-realizable} if there is $h\in H$ such that $h(x_i)=y_i$ for all $i$.  Repetitions of domain points are allowed, but realizability forces repeated occurrences of the same point to carry the same label.  We write $(x,b)\in S$ if the signed point $(x,b)$ occurs as an entry of $S$, and
\[
\supp(S)=\{x_i:i\in[m]\}.
\]

\begin{definition}[Contradiction graph~\cite{AlonMoranScheflerYehudayoff2024}]
The \emph{order-$m$ contradiction graph} $G_m(H)$ has as vertices the $H$-realizable length-$m$ labeled sequences.  Two distinct vertices $S,T$ are adjacent if there is a domain point $x\in X$ and a bit $b\in\bits$ such that
\[
(x,b)\in S
\qquad\text{and}\qquad
(x,1-b)\in T.
\]
\end{definition}

We also consider a duplicate-free variant. In this model, datasets with repeated examples are no longer vertices.

\begin{definition}[Duplicate-free contradiction graph]\label{def:df-graph}
The \emph{duplicate-free order-$m$ contradiction graph} $G_m^{\df}(H) = G_m(H)[V']$ is the subgraph of $G_m(H)$ induced by the set $V' \subset V(G_m(H))$ of vertices 

$$V' = \{ S \in V : |\text{supp}(S)| = m \}.$$
\end{definition}

At each level, a permutation of domain points preserves contradiction, as does independently flipping the two labels above each domain point. This gives the following equivalence.

\begin{definition}[Learning-equivalence]\label{def:signed-equivalence}
Let $H\subseteq\bits^X$ and $K\subseteq\bits^Y$ be nonempty concept classes. Given a bijection $\pi:X\to Y$ and a map $\epsilon:X\to\bits$, define the signed relabeling of a partial assignment $p:D\to\bits$ by
\[
\Theta_{\pi,\epsilon}(p):\pi(D)\longrightarrow\{0,1\},\qquad \pi(x)\longmapsto p(x)\oplus\epsilon(x),
\]
and apply the same map entrywise to labeled sequences. We say that $H$ and $K$ are \emph{learning-equivalent}, and write $H\equiv_{\mathrm{le}}K$, if for every $m\geq1$ there is a bijection $\pi_m:X\to Y$ and a map $\epsilon_m:X\to\bits$ such that, for every length-$m$ labeled sequence $S$ over $X$,
\[
S\text{ is $H$-realizable}
\quad\Longleftrightarrow\quad
\Theta_{\pi_m,\epsilon_m}(S)\text{ is $K$-realizable}.
\]
\end{definition}

We call the equivalence relation $\equiv_{\text{le}}$ \textit{learning-equivalence} as it preserves all learning-theoretic properties invariant under signed relabelings and defined in terms of realizable datasets of a fixed sample size. This includes most relevant properties such as VC dimension and Littlestone dimension. 

Lastly, a point that receives the same label from every hypothesis produces no contradictions, and is therefore invisible to adjacency in the contradiction graph. We call such an example \textit{boring}.

\begin{definition}\label{def:boring}
Let $H\subseteq\bits^X$ be nonempty.  A point $x\in X$ is \emph{boring} for $H$ if $h(x)$ is the same for every $h\in H$. A point is \textit{interesting} if it is not boring.  Let
\[
\Boring(H)=\{x\in X:x\text{ is boring for }H\},
\qquad
\Int(H)=X\setminus \Boring(H).
\]
\end{definition}


\subsubsection*{Boolean hypercubes}

\begin{definition}
A \emph{Boolean subcube} of $\bits^m$ is a nonempty set of the form
\[
\{\sigma\in\bits^m:\sigma_i=\tau_i\text{ for all }i\in I\},
\]
where $I\subseteq[m]$ and $\tau\in\bits^I$.  The case $I=\emptyset$ gives the full cube, and the case $I=[m]$ gives a singleton.
\end{definition}

\subsubsection*{Graph theory}

In this paper we use two graph-theoretic properties which, when viewed in $G_m(H)$, are related to learning-theoretic notions of $H$: \textit{cube-trace cliques} and \textit{infinite neighborhoods}

We define cube-trace cliques based on the adjacency behavior of vertices outside the clique to vertices inside the clique, namely, the anti-neighborhood of a vertex.

\begin{definition}
Let $G = (V, E)$ be a simple graph and let $Q\subseteq V$.  For $T\in V$, the \emph{anti-neighborhood} of $T$ on $Q$ is
\[
M_Q(T)=\{S\in Q:S\text{ is not adjacent to }T\}.
\]
Since the graph is simple, a vertex is not adjacent to itself; thus if $T\in Q$ and $Q$ is a clique, then $M_Q(T)=\{T\}$.
\end{definition}

Infinite neighborhoods are used for the learning-equivalence proof.

\begin{definition}
\label{def:inf_neigh}
For a simple graph $G = (V,E)$ and $A\subseteq V$, define the \textit{infinite neighborhood} of $A$ as,
\[
N_\infty(A)
=
\{T\in V:|N_G(T)\cap A|=\infty\}.
\]
\end{definition}

By taking a minimal infinite neighborhood in the duplicate-free graph, which we call an \textit{example set}, we identify the set of vertices containing a particular signed example.

\begin{definition}
\label{def:exam_set}
A set $\mathcal V\subseteq V(G_{m}^{\text{df}}(H))$ is an \emph{example set} if
\[
\mathcal V=N_\infty(A)
\]
for some $A\subseteq V(G_m^{\text{df}}(H))$, $\mathcal V\neq\emptyset$, and $\mathcal V$
is inclusion-minimal among the nonempty sets obtainable in this way.
\end{definition}

\section{Cube-trace cliques}\label{sec:cubetrace}

We prove that $G_m(H)$ determines the predicate $\VC(H) \geq m$ via the presence of \textit{cube-trace cliques}. Whereas every shattered set of $m$ examples gives rise to a clique of size $2^m$ in $G_m(H)$, the converse is not true. A clique that corresponds to a shattered set is much more structured. If an outside sample fixes some of the shattered coordinates, then the clique vertices nonadjacent to it are exactly those cube labels agreeing on those fixed coordinates.  This is precisely a Boolean subcube.  The following definition formalizes this adjacency behavior.

\begin{definition}[Cube-trace clique]\label{def:cubetrace}
Let $m\ge1$, and let $Q\subseteq V(G_m(H))$ be a clique of size $2^m$.  We say that $Q$ is a \emph{cube-trace clique} if there is a bijection
$
\phi:Q\to\bits^m
$
such that, for every vertex $T\in V(G_m(H))$, the set
$
\phi(M_Q(T))\subseteq\bits^m
$
is a Boolean subcube.
\end{definition}

This definition is purely graph-theoretic: it only refers to adjacency and non-adjacency in $G_m(H)$, together with a labeling of the clique by the boolean cube $\bits^m$. Because Boolean subcubes are taken to be nonempty, a cube-trace clique cannot have a vertex $T$ with $M_Q(T)=\emptyset$, that is, a vertex adjacent to every member of $Q$. Figure~\ref{fig:genuine-cube-trace} illustrates the cube-trace clique arising from a shattered two-point set. The clique vertices nonadjacent to $T$ form the Boolean subcube $\{00,01\}$.

\begin{proposition}\label{prop:forward}
If $\VC(H)\ge m$, then $G_m(H)$ contains a cube-trace clique of size $2^m$.
\end{proposition}

\begin{proof}
Let $R=\{x_1,\ldots,x_m\}\subseteq X$ be shattered by $H$.  For each $\sigma\in\bits^m$, define the length-$m$ sequence
$
S_\sigma=((x_1,\sigma_1),\ldots,(x_m,\sigma_m)).
$
Since $R$ is shattered, every $S_\sigma$ is $H$-realizable.  Let
\[
Q_R=\{S_\sigma:\sigma\in\bits^m\}.
\]
If $\sigma\ne\tau$, then $\sigma_i\ne\tau_i$ for some $i$, so $S_\sigma$ and $S_\tau$ contradict at $x_i$.  Hence $Q_R$ is a clique of size $2^m$.

Identify $Q_R$ with $\bits^m$ by $S_\sigma\mapsto\sigma$.  Let $T$ be any vertex of $G_m(H)$.  Let $I\subseteq[m]$ be the set of indices $i$ such that $x_i\in\supp(T)$.  For each $i\in I$, realizability of $T$ gives a unique bit $\tau_i$ such that $(x_i,\tau_i)\in T$.  A cube vertex $S_\sigma$ is nonadjacent to $T$ exactly when it agrees with $T$ on all common points $x_i$, namely when
$
\sigma_i=\tau_i$ for every $i\in I$.

Therefore
\[
M_{Q_R}(T)=\{S_\sigma:\sigma_i=\tau_i\text{ for all }i\in I\},
\]
which is a Boolean subcube under the identification $Q_R\cong\bits^m$.
\end{proof}

\begin{figure}[t]
\centering
\begin{tikzpicture}[
    scale=1,
    every node/.style={font=\small},
    cubev/.style={circle,draw,thick,minimum size=8mm,inner sep=0pt},
    subcube/.style={circle,draw,thick,fill=gray!20,minimum size=8mm,inner sep=0pt},
    outside/.style={rectangle,draw,thick,rounded corners,minimum width=13mm,minimum height=8mm},
    edge/.style={thick},
    nonedge/.style={thick,dashed},
    ann/.style={font=\footnotesize}
]

\node[subcube] (00) at (0,0) {$00$};
\node[subcube] (01) at (0,2) {$01$};
\node[cubev]   (10) at (2,0) {$10$};
\node[cubev]   (11) at (2,2) {$11$};

\draw[edge] (00)--(01);
\draw[edge] (00)--(10);
\draw[edge] (00)--(11);
\draw[edge] (01)--(10);
\draw[edge] (01)--(11);
\draw[edge] (10)--(11);

\node[outside] (T) at (-3.2,1) {$T$};

\draw[nonedge] (T)--(00);
\draw[nonedge] (T)--(01);
\draw[edge] (T)--(10);
\draw[edge] (T)--(11);

\node[
    draw,
    dashed,
    rounded corners,
    fit=(00)(01),
    inner sep=6pt
] {};

\node[ann,align=center] at (1,-1.0)
    {$Q_R \cong \{0,1\}^2$};

\node[ann,align=center] at (-1.05,2.75)
    {$M_{Q_R}(T)=\{00,01\}$};

\node[ann,align=center] at (-3.2,1.65)
    {$T$ fixes $x_1=0$};

\end{tikzpicture}
\caption{A shattered two-point set produces a cube-trace clique. The shaded vertices are precisely the members of the clique nonadjacent to $T$; they form the Boolean subcube $\{00,01\}$. Solid lines indicate adjacency, and dashed lines indicate non-adjacency.}
\label{fig:genuine-cube-trace}
\end{figure}

Proposition~\ref{prop:forward} explains the origin of the definition of cube-trace cliques: shattered sets do not only create large cliques, they create large cliques whose traces are coordinate subcubes. The proof of the converse uses the next elementary fact, which is the only Boolean-cube geometry needed. Namely, if two proper subcubes cover the entire cube, then they must be opposite half-cubes.

\begin{lemma}\label{lem:two-subcubes}
Let $m\ge1$, and let $A,B\subseteq\bits^m$ be proper Boolean subcubes such that
$
A\cup B=\bits^m.
$
Then $A$ and $B$ are complementary facets: there is an $i\in[m]$ and $b\in\bits$ such that
\[
A=\{\sigma\in\bits^m:\sigma_i=b\},
\qquad
B=\{\sigma\in\bits^m:\sigma_i=1-b\}.
\]
\end{lemma}

\begin{proof}
Every proper Boolean subcube of $\bits^m$ has size at most $2^{m-1}$.  Since $A\cup B=\bits^m$,
\[
2^m=|A\cup B|\le |A|+|B|\le 2^{m-1}+2^{m-1}=2^m.
\]
Thus equality holds throughout.  In particular, both $A$ and $B$ have size $2^{m-1}$, so each is a codimension-one subcube, i.e. a facet.  The equality also implies $A\cap B=\emptyset$.  Two facets of a Boolean cube are disjoint only when they fix the same coordinate to opposite bits.  Hence they are complementary facets.
\end{proof}

\begin{theorem}\label{thm:reverse}
If $G_m(H)$ contains a cube-trace clique of size $2^m$, then $\VC(H)\ge m$.
\end{theorem}

\begin{proof}
Let $Q\subseteq V(G_m(H))$ be a cube-trace clique, and fix a witnessing bijection
$
\phi:Q\to\bits^m.
$
For a domain point $x\in X$ and a bit $b\in\bits$, define
$
Q_x^b=\{S\in Q:(x,b)\in S\}.
$
Call $x$ \emph{informative} if both $Q_x^0$ and $Q_x^1$ are nonempty.

Fix an informative point $x$.  Since some vertex of $Q$ contains $(x,0)$ and some vertex of $Q$ contains $(x,1)$, both labels of $x$ are realized by concepts in $H$.  Hence the repeated sequences
\[
T_x^0=((x,0),\ldots,(x,0)),
\qquad
T_x^1=((x,1),\ldots,(x,1))
\]
are vertices of $G_m(H)$. 

A vertex $S\in Q$ is adjacent to $T_x^0$ exactly when $(x,1)\in S$.  Therefore
$
M_Q(T_x^0)=Q\setminus Q_x^1.
$
Similarly,
$
M_Q(T_x^1)=Q\setminus Q_x^0.
$
By cube-traceness, the images
\[
A:=\phi(M_Q(T_x^0)),
\qquad
B:=\phi(M_Q(T_x^1))
\]
are Boolean subcubes of $\bits^m$.  Since $x$ is informative, both $Q_x^0$ and $Q_x^1$ are nonempty, so $A$ and $B$ are proper subcubes.  Also, no realizable sequence can contain both $(x,0)$ and $(x,1)$, so $Q_x^0\cap Q_x^1=\emptyset$.  Hence
\[
M_Q(T_x^0)\cup M_Q(T_x^1)
=(Q\setminus Q_x^1)\cup(Q\setminus Q_x^0)
=Q.
\]
Thus $A\cup B=\bits^m$.  By Lemma~\ref{lem:two-subcubes}, $A$ and $B$ are complementary facets.  Taking complements in $\bits^m$, the sets $\phi(Q_x^0)$ and $\phi(Q_x^1)$ are also complementary facets.  In particular,
$
Q_x^0\cup Q_x^1=Q.
$
So every informative point appears in every vertex of $Q$.

Let $B_0\subseteq X$ be the set of informative points.  Since every informative point appears in every vertex of $Q$, for any fixed $S_0\in Q$ we have
$
B_0\subseteq \supp(S_0),
$
and therefore $|B_0|\le m$.
Now take two distinct vertices $S,S'\in Q$.  Because $Q$ is a clique, $S$ and $S'$ contradict at some domain point $x$.  That point is informative, since one of $S,S'$ contains $(x,0)$ and the other contains $(x,1)$.  Therefore $S$ and $S'$ induce different labelings on $B_0$.  The map
\[
S\longmapsto \bigl(\text{the labels assigned by }S\text{ to the points of }B_0\bigr)
\]
is injective from $Q$ into $\bits^{B_0}$.  Consequently
$
2^m=|Q|\le 2^{|B_0|},
$
so $|B_0|\ge m$.  Combining this with $|B_0|\le m$ gives $|B_0|=m$.

Every vertex of $Q$ contains every point of $B_0$, and $|B_0|=m$.  Since each vertex is a sequence of length $m$, each vertex uses exactly the support $B_0$, with each point appearing exactly once.  The injective map $Q\to\bits^{B_0}$ is now a bijection, because both sets have size $2^m$.  Hence the vertices of $Q$ realize all $2^m$ labelings of $B_0$.  Therefore $B_0$ is shattered by $H$, and $\VC(H)\ge m$.
\end{proof}

Combining Proposition~\ref{prop:forward} and Theorem~\ref{thm:reverse} gives the graph-intrinsic characterization.

\begin{corollary}\label{cor:level-m}
For every binary class $H$ and every $m\ge1$,
$
\VC(H)\ge m$
if and only if
$G_m(H)$ contains a cube-trace clique of size $2^m$.
\end{corollary}

\begin{proof}
The forward implication is Proposition~\ref{prop:forward}, and the reverse implication is Theorem~\ref{thm:reverse}.
\end{proof}


Corollary~\ref{cor:level-m} is expressed entirely in the language of the graph $G_m(H)$: it asserts the existence of a clique of size $2^m$ with a specified trace pattern.  Therefore it is preserved by graph isomorphism.  This turns the structural characterization into the following statement.

\begin{theorem}\label{thm:main}
For every binary concept class $H$ and every $m\ge1$, the single graph $G_m(H)$ determines whether $\VC(H)\ge m$.  Equivalently, if
$
G_m(H)\cong G_m(K),
$
then
$
\VC(H)\ge m$
if and only if
$\VC(K)\ge m.
$
\end{theorem}
\begin{proof}
Let $G_m(H)\cong G_m(K)$, and suppose $\VC(H)\ge m$.  By Corollary~\ref{cor:level-m}, $G_m(H)$ contains a cube-trace clique $Q$ of size $2^m$.  Let
$
f:G_m(H)\to G_m(K)
$
be a graph isomorphism, and let $\phi:Q\to\bits^m$ witness cube-traceness of $Q$.  We claim that
$
\phi\circ f^{-1}:f(Q)\to\bits^m
$
witnesses cube-traceness of $f(Q)$.  Indeed, for every $T'\in V(G_m(K))$,
$
M_{f(Q)}(T')=f\bigl(M_Q(f^{-1}(T'))\bigr),
$
because $f$ preserves adjacency and non-adjacency.  Therefore
\[
(\phi\circ f^{-1})(M_{f(Q)}(T'))
=\phi\bigl(M_Q(f^{-1}(T'))\bigr),
\]
which is a Boolean subcube by the cube-traceness of $Q$.  Hence $G_m(K)$ contains a cube-trace clique of size $2^m$, so Corollary~\ref{cor:level-m} gives $\VC(K)\ge m$.  The reverse implication is symmetric.
\end{proof}

\begin{corollary}\label{cor:exact}
If $G_m(H)\cong G_m(K)$ for every $m\ge1$, then
$
\VC(H)=\VC(K).
$
\end{corollary}

\begin{proof}
By Theorem~\ref{thm:main}, for every $m\ge1$,
$\VC(H)\ge m$
if and only if $\VC(K)\ge m.$
Thus the two classes have exactly the same positive VC-thresholds.  If there are arbitrarily large such thresholds, both VC dimensions are infinite.  Otherwise, the largest positive threshold is the same for both classes; if there is no positive threshold, both VC dimensions are $0$ by the convention in the definition of VC dimension.  Hence $\VC(H)=\VC(K)$ in $\mathbb N_0\cup\{\infty\}$.
\end{proof}

\begin{corollary}\label{cor:amsy}
The pair of classes in Problem~\ref{prob:amsy} does not exist.
\end{corollary}

\begin{proof}
If such $H_0,H_1$ existed, Corollary~\ref{cor:exact} would give $\VC(H_0)=\VC(H_1)$, contradicting $\VC(H_0)<\infty$ and $\VC(H_1)=\infty$.
\end{proof}

We end by remarking that, in infinite domains, that is, when $|X| = \infty$, we may replace a \textit{probe} $((x,0),...,(x,0))$ in the proof of Theorem \ref{thm:reverse} with a vertex $((x,0), (z_1,h(z_1)), ..., (z_{m-1}, h(z_{m-1})))$ where $z_i \not\in \cup_{S \in Q} \ \text{supp}(S)$ are distinct for all $i \in[m-1]$ and $h\in H$ is any concept realizing $h(x) = 0$. As this is the only usage of datasets containing repeated examples in the argument, the conclusion Corollary \ref{cor:exact} also holds for duplicate-free graphs in infinite domains. In the next section, we prove that this same behavior holds for a broader class of learning-theoretic properties including VC dimension.

\section{Learning-equivalence from contradiction graphs}\label{sec:reconstruction}

In this section, we prove that the sequence of contradiction graphs $(G_m(H))_{m\geq1}$ determines the datasets realizable by $H$ at every fixed level up to learning-equivalence (Definition~\ref{def:signed-equivalence}). Thus, all learning-theoretic properties determined by realizable samples of a fixed size and invariant under signed relabeling are captured by the contradiction-graph sequence. As an example, we demonstrate in Section~\ref{sec:littlestone} an exact graph-theoretic characterization of Littlestone dimension.

\begin{theorem}\label{thm:reconstruction-main}
Let $H\subseteq\bits^X$ and $K\subseteq\bits^Y$ be nonempty concept classes. Then the following are equivalent.
\begin{enumerate}[label=(\roman*)]
\item $G_m(H)\cong G_m(K)$ for every $m\geq1$.
\item $H\equiv_{\mathrm{le}}K$.
\end{enumerate}
\end{theorem}

As aforementioned, the result of Theorem \ref{thm:reconstruction-main} also holds for duplicate-free graphs $G_m^{\text{df}}(H)$ (Definition \ref{def:df-graph}) when we assume the graph is defined over an infinite domain.

\begin{theorem}\label{thm:reconstruction-df}
Let $H\subseteq\bits^X$ and $K\subseteq\bits^Y$ be nonempty concept classes with $X,Y$ infinite. Then the following are equivalent.
\begin{enumerate}[label=(\roman*)]
\item $G_m^{\text{df}}(H)\cong G_m^{\text{df}}(K)$ for every $m\geq1$.
\item $H\equiv_{\mathrm{le}}K$.
\end{enumerate}
\end{theorem}


We begin by proving Theorem \ref{thm:reconstruction-df} and then show how Theorem \ref{thm:reconstruction-main} follows by a minor modification to the proof. The implication $(ii)\Rightarrow(i)$ is straightforward: at each level, a signed relabeling preserves contradiction. 

To prove the converse, we demonstrate a graph-theoretic property which captures, for each interesting signed example, the set of datasets containing it. Once these sets are identified, the vertex labels of the entire graph can easily be reproduced by propagating the known labels throughout the graph.

Let $H$ be a concept class over a domain $X$. Unless stated otherwise, we assume that $X$ is infinite, and let $m\geq 2$. For a labeled example $e$, define,
\[
\mathcal V_e=\{S\in V(G_m^{\text{df}}(H)):e\in S\}.
\]
$\mathcal V_e$ is the collection of all realizable duplicate-free length-$m$ datasets containing $e$. We observe that each infinite neighborhood (Definition \ref{def:inf_neigh}) is a union of sets of the form $\mathcal{V}_{e}$.

\begin{lemma}\label{lem:infinite-neighborhood-union}
For every $A\subseteq V(G_m^{\df}(H))$,
\[
N_\infty(A)
=
\bigcup_{\substack{e\text{ a signed example}:\\
|\{S\in A:e\in S\}|=\infty}}
\mathcal V_{\overline e}.
\]
\end{lemma}

\begin{proof}
Fix $T\in V(G_m^{\df}(H))$.  A vertex $S\in A$ is adjacent to $T$ if and only if
$S$ contains the negation of at least one signed example appearing in $T$.
Consequently,
\[
N_G(T)\cap A
=
\bigcup_{e\in T}
\{S\in A:\overline e\in S\}.
\]
Since $T$ contains exactly $m$ signed examples, the RHS is a union of
$m$ sets. It is infinite if and only if at least one of these sets is
infinite. Equivalently,
\(
T\in N_\infty(A)
\)
if and only if $T$ contains $\overline e$ for some signed example $e$
that occurs in infinitely many members of $A$.  This is exactly the
claimed identity.
\end{proof}

We may then isolate the individual sets $\mathcal{V}_e$ as minimal infinite neighborhoods, or example sets (Definition \ref{def:exam_set}).

\begin{lemma}\label{lem:abstract-signed-examples}
The example sets in $G_m^{\mathrm{df}}(H)$ are exactly the sets
\[
\mathcal V_{(x,b)}
\]
for which $x$ is interesting.
\end{lemma}

\begin{proof}
We first show that every $\mathcal V_{(x,b)}$ with $x$ interesting is an
infinite neighborhood.  Since $x$ is interesting, there is some $h\in H$
such that
\[
h(x)=1-b.
\]
Choose infinitely many pairwise disjoint sets
\[
P_1,P_2,\ldots\subseteq X\setminus\{x\},
\qquad |P_i|=m-1,
\]
and let
\[
S_i=h|_{\{x\}\cup P_i}.
\]
Set $A=\{S_i:i\geq 1\}$.  The signed example $(x,1-b)$ occurs in every member of $A$, while every other point occurs in at most one member. Lemma~\ref{lem:infinite-neighborhood-union} therefore gives
\[
N_\infty(A)=\mathcal V_{(x,b)}.
\]

We next observe that distinct nonempty sets of the form $\mathcal V_e$
are incomparable under inclusion.  Suppose $e=(x,b)$ and $f=(y,c)$ are
distinct and both $\mathcal V_e$ and $\mathcal V_f$ are nonempty.  If
$x=y$, then $b\neq c$, so the two sets are disjoint.  If $x\neq y$, choose
a hypothesis realizing $e$ and restrict it to an $m$-point set containing
$x$ but not $y$.  This gives a vertex in $\mathcal V_e\setminus\mathcal
V_f$.  By symmetry, neither set contains the other.

Now let $N_\infty(A)$ be nonempty.  By
Lemma~\ref{lem:infinite-neighborhood-union}, it is a union of nonempty
sets $\mathcal V_e$.  Each such $\mathcal V_e$ is itself an infinite
neighborhood by the first part of the proof.  Therefore, if
$N_\infty(A)$ is inclusion-minimal and nonempty, it must be equal to a single
$\mathcal V_e$.  Conversely, incomparability shows that no nonempty
infinite neighborhood can be a proper subset of $\mathcal V_e$.
Thus the minimal nonempty infinite neighborhoods are precisely the
claimed sets.
\end{proof}


It is then also straightforward to detect pairs of vertex sets containing oppositely labeled examples.

\begin{definition}
Two distinct example sets $\mathcal V_0$ and $\mathcal V_1$ are \emph{complementary} if every vertex in $\mathcal V_0$ is adjacent to every vertex in $\mathcal V_1$.
\end{definition}

\begin{lemma}\label{lem:complementary-abstract-examples}
Two example sets are complementary if and only if they are
of the form
\[
\mathcal V_{(x,0)}
\qquad\text{and}\qquad
\mathcal V_{(x,1)}
\]
for some interesting point $x$.
\end{lemma}

\begin{proof}
Every vertex containing $(x,0)$ contradicts every vertex containing
$(x,1)$, so these two sets are complementary.

Conversely, suppose $\mathcal V_{(x,b)}$ and $\mathcal V_{(y,c)}$ are
complementary.  If $x\neq y$, use the infinitude of $X$ to choose a
realizable vertex from each set with disjoint supports.  These two
vertices are nonadjacent, a contradiction.  Hence $x=y$.  Since the two
example sets are distinct, their labels must be opposite.
\end{proof}


Next, we show how to recover all labels in the duplicate-free contradiction graph $G_m^{\df}(H)$ up to a signed relabeling. Once we have identified, for each signed, interesting example, the set of vertices containing that example, the full vertex labels follow easily by adjacency to an example labeled oppositely.


\begin{lemma}\label{lem:fixed-level-reconstruction}
Let $H\subseteq\bits^X$ and $K\subseteq\bits^Y$ be nonempty concept classes with $X,Y$ infinite, let $m\geq2$, and let
\[
f:G^{\text{df}}_m(H)\longrightarrow G^\text{df}_m(K)
\]
be a graph isomorphism. Then there is a bijection
\[
\pi_f:\Int(H)\to\Int(K)
\]
and a map $\epsilon_f:\Int(H)\to\bits$ such that, for every partial assignment $p:D\to\bits$ with $D\subseteq\Int(H)$ and $|D|\leq m$,
\[
p\text{ is realized by }H
\quad\Longleftrightarrow\quad
\Theta_{\pi_f,\epsilon_f}(p)\text{ is realized by }K.
\]
\end{lemma}

\begin{proof}
Consider an interesting labeled example $e=(x,b)$. Since example sets are purely graph-theoretic objects, the isomorphism $f$ preserves them. Therefore, by Lemma \ref{lem:abstract-signed-examples}, we have
\[
f(\mathcal V_e)=\mathcal V_{e'}
\]
for some signed example $e'=(x',b')\in Y\times\bits$. Now, by Lemma \ref{lem:complementary-abstract-examples}, we also observe that $f(\mathcal{V}_{\bar e}) = \mathcal V_{\overline{e'}}$. In particular, $e'$ is also an interesting example of $K$. 


Thus, the isomorphism $f$ induces a bijection $\theta_f$ between the signed interesting examples of $H$ and those of $K$, satisfying
\[
\theta_f(\bar e)=\overline{\theta_f(e)}.
\]
Equivalently, there is a bijection $\pi_f:\Int(H)\to\Int(K)$ and a map $\epsilon_f:\Int(H)\to\bits$ such that
\[
\theta_f(x,b)=\bigl(\pi_f(x),b\oplus\epsilon_f(x)\bigr).
\]

For a vertex $S\in V(G_m^{\mathrm{df}}(H))$, let
\[
P_H(S)=\{(x,b)\in S:x\in\Int(H)\}.
\]
For every interesting signed example $e$, we have
\[
e\in P_H(S)
\quad\Longleftrightarrow\quad
S\in\mathcal V_e
\quad\Longleftrightarrow\quad
f(S)\in\mathcal V_{\theta_f(e)}
\quad\Longleftrightarrow\quad
\theta_f(e)\in P_K(f(S)).
\]
Consequently,
\[
\theta_f(P_H(S))=P_K(f(S))
\qquad\text{for every }S\in V(G^{\text{df}}_m(H)).
\]

Finally, let $p:D\to\bits$ be a partial assignment with $D\subseteq\Int(H)$ and $|D|\leq m$. Suppose that $p$ is realized by some $h\in H$. Since $X$ is infinite, choose a set $P\subseteq X\setminus D$ of size $m-|D|$, and let $S$ be any ordering of the labeled set $h|_{D\cup P}$. Then $S\in V(G_m^{\mathrm{df}}(H))$ and
\[
\{(x,p(x)):x\in D\}\subseteq P_H(S).
\]
It follows from the preceding identity that $P_K(f(S))$ contains the signed examples of $\Theta_{\pi_f,\epsilon_f}(p)$. Since $f(S)$ is $K$-realizable, so is $\Theta_{\pi_f,\epsilon_f}(p)$. The reverse implication follows by applying the same argument to $f^{-1}$, which induces the inverse signed bijection. This also covers the empty assignment.
\end{proof}

Next, we prove the main theorem.

\begin{proof}[Proof of Theorem~\ref{thm:reconstruction-df}]
Assume first that $H\equiv_{\mathrm{le}}K$. For each $m$, applying the relabeling witnessing $H\equiv_{\mathrm{le}}K$ gives a bijection between the vertex sets of $G_m^\text{df}(H)$ and $G_m^\text{df}(K)$. Since signed relabelings preserve contradiction, this bijection is a graph isomorphism.

Conversely, suppose $G^\text{df}_m(H)\cong G^\text{df}_m(K)$ for every $m\geq1$. The graph $G^\text{df}_1(H)$ is the disjoint union of one edge for every interesting point and one isolated vertex for every boring point. Hence $G^\text{df}_1(H)\cong G^\text{df}_1(K)$ implies
\[
|\Int(H)|=|\Int(K)|
\qquad\text{and}\qquad
|\Boring(H)|=|\Boring(K)|.
\]
Fix a bijection $\beta:\Boring(H)\to\Boring(K)$. For $x\in\Boring(H)$ and $y\in\Boring(K)$, let $b_H(x)$ and $b_K(y)$ denote their unique realizable labels, and define
\[
\epsilon_B(x)=b_H(x)\oplus b_K(\beta(x)).
\]

For $m=1$, an isomorphism $G_1^{\mathrm{df}}(H)\cong G_1^{\mathrm{df}}(K)$ maps edge components to edge components and therefore induces a signed bijection $(\pi_1,\epsilon_1)$ from $\Int(H)$ to $\Int(K)$. Extend it to the boring points using $\beta$ and $\epsilon_B$. The resulting signed relabeling preserves the realizability of every length-$1$ sequence.

Fix $m\geq2$ and an isomorphism $f_m:G^\text{df}_m(H)\to G^\text{df}_m(K)$. By Lemma~\ref{lem:fixed-level-reconstruction}, it induces a signed bijection $(\pi_m,\epsilon_m)$ from $\Int(H)$ to $\Int(K)$ preserving every realizable partial assignment on at most $m$ interesting points. Extend these maps to all of $X$ by setting
\[
\pi_m|_{\Boring(H)}=\beta,
\qquad
\epsilon_m|_{\Boring(H)}=\epsilon_B.
\]
A labeled sequence is realizable precisely when repeated occurrences of each point carry a consistent label, every boring point receives its unique realizable label, and the induced partial assignment on its interesting support is realizable. The signed relabeling preserves the first two conditions, while Lemma~\ref{lem:fixed-level-reconstruction} preserves the third. Therefore, for every length-$m$ sequence $S$ over $X$,
\[
S\text{ is $H$-realizable}
\quad\Longleftrightarrow\quad
\Theta_{\pi_m,\epsilon_m}(S)\text{ is $K$-realizable}.
\]
Together with the $m=1$ case above, this shows that $H$ and $K$ are learning-equivalent.
\end{proof}

\medskip

The only remaining question is whether learning-equivalence is preserved by duplicate-free graphs in the finite domain setting. As the following example shows, there is no hope for a finite-case theorem like the one above.

\begin{example}\label{ex:df-not-reconstructive}
Let the domain be $\{1,2,3\}$, and define
\[
H=\bits^3\setminus\{101,110\},
\qquad
K=\bits^3\setminus\{000,111\}.
\]
The classes $H$ and $K$ are not learning-equivalent. To see this, at level $m=3$, a signed relabeling preserving realizability would carry the six full labelings in $H$ onto the six full labelings in $K$. Signed relabelings of the boolean cube are exactly the isometries generated by coordinate permutations and coordinate flips, and hence preserve the Hamming distance between the two omitted concepts. The omitted concepts of $H$ have distance $2$, while the omitted concepts of $K$ have distance $3$.

Nevertheless,
\[
G_m^{\df}(H)\cong G_m^{\df}(K)
\qquad\text{for every }m\ge1.
\]
For $m=1$ and $m=2$, every duplicate-free labeled sample of size $m$ is realizable in both classes, so the graphs are identical. For $m=3$, each realized full labeling contributes the $3!$ orderings of that labeling. These orderings form an independent set, and every two orderings of distinct full labelings are adjacent. Hence both graphs are the complete six-partite graph with parts of size $3!$. For $m>3$, both graphs are empty.
\end{example}


Now that we have derived all the results for the duplicate-free contradiction graph, obtaining an analogous result for the ordinary contradiction graph follows nearly identically. However, in the contradiction graph with repetitions, minimal vertices identify the example sets without requiring an infinite domain. We provide a proof for completeness.

\begin{proof}[Proof of Theorem \ref{thm:reconstruction-main}]
The proof is identical to the proof of Theorem \ref{thm:reconstruction-df}, with two minor modifications.

First, the assumption that $X$ and $Y$ are infinite is unnecessary. Fix $m\geq2$, and for an interesting signed example $e$ put
\[
\mathcal V_e=\{S\in V(G_m(H)):e\in S\}.
\]
By Lemma~\ref{lem:minimal-one-moving} and the construction immediately following it, the equivalence classes of minimal vertices under equality of neighborhoods are precisely the classes $\mathcal A_e$ indexed by interesting signed examples. Moreover, two such classes are completely adjacent exactly when they are $\mathcal A_e$ and $\mathcal A_{\bar e}$. These two classes are completely adjacent, while if $f\ne\bar e$, the vertices $e^m\in\mathcal A_e$ and $f^m\in\mathcal A_f$ are nonadjacent. Also,
\[
\mathcal V_e=\{S\in V(G_m(H)):S\text{ is adjacent to every }U\in\mathcal A_{\bar e}\}.
\]
Thus the sets $\mathcal V_e$ and their complementary pairing are determined by the graph, so the argument of Lemma~\ref{lem:fixed-level-reconstruction} applies. In its final padding step, a nonempty realized partial assignment is extended to a length-$m$ vertex by repeating any one of its signed examples; the empty assignment is realized automatically. For $m=1$, the edge-component argument used in the proof of Theorem~\ref{thm:reconstruction-df} applies unchanged.

Second, in extending the signed bijection from interesting points to the entire domain, repeated examples cause no additional difficulty, that is, the signed relabeling acts on each occurrence of a point in the same way, so repeated occurrences are automatically assigned consistently. Thus, the argument following Lemma \ref{lem:fixed-level-reconstruction} in the proof of Theorem \ref{thm:reconstruction-df} applies verbatim. In particular, an isomorphism
\[
G_m(H)\cong G_m(K)
\]
induces a signed bijection of the interesting points that preserves every realizable partial assignment on at most $m$ interesting points. The bijection between boring points is obtained from the isomorphism of $G_1(H)$ and $G_1(K)$ exactly as in Theorem \ref{thm:reconstruction-df}.

Hence, for every $m\geq1$, the resulting signed relabeling preserves the realizability of every length-$m$ labeled sequence, and therefore $H\equiv_{\mathrm{le}}K$. The reverse implication is immediate, since signed relabelings preserve contradiction. This proves the theorem.
\end{proof}

Lastly, we take a short detour to explain how one could obtain the vertex sets $\mathcal{V}_e$ in $G_m(H)$ through a much simpler, though less general, construction. This relies on the fact that, in the contradiction graph with repetitions, we can identify the points $((x,b),\dots,(x,b))$, as well as any other points with only one interesting example based on the notion of minimality described below.

\begin{definition}\label{def:minimal}
Let $G=(V,E)$ be a graph, and let $\Nbr(v)$ denote the open neighborhood of $v$. A vertex $v\in V$ is \emph{minimal} if $\Nbr(v)\ne\emptyset$ and there is no $u\in V$ such that
\[
\emptyset\subsetneq \Nbr(u)\subsetneq \Nbr(v).
\]
\end{definition}

The following lemma establishes the desired characterization; its proof is given in Appendix \ref{apx:proofs}.

\begin{lemma}\label{lem:minimal-one-moving}
Let $H\subseteq\bits^X$ be nonempty, and let $S\in V(G_m(H))$. Then $S$ is minimal if and only if
\[
|\supp(S)\cap\Int(H)|=1.
\]
\end{lemma}

With this result, one can easily recover the sets $\mathcal V_e$ in the following manner. For an interesting signed example $e=(x,b)$, let $\mathcal A_e$ be the set of minimal vertices whose unique interesting signed example is $e$. This set is nonempty, since it contains $e^m$. By Lemma~\ref{lem:minimal-one-moving}, every minimal vertex belongs to exactly one such set, and every $U\in\mathcal A_e$ satisfies
\[
\Nbr(U)=\Nbr(e^m).
\]
Moreover, if $e\neq e_0$, then
\[
\bar e^m\in\Nbr(e^m)\setminus\Nbr(e_0^m).
\]
Consequently, the sets $\mathcal A_e$ are precisely the equivalence classes of minimal vertices under equality of neighborhoods. In particular, the sets $\mathcal A_e$ are definable in purely graph-theoretic terms. Now, the sets $\mathcal V_e$ are exactly the sets of vertices $S$ that are connected to every element of $\mathcal A_{\bar e}$. Accordingly, the sets $\mathcal V_e$ are the commensurate \emph{example sets} of $G_m(H)$.

\subsection{Characterizing Littlestone dimension}
\label{sec:littlestone}

Alon et al. \cite{AlonMoranScheflerYehudayoff2024} define the \textit{clique dimension} of $(G_m(H))_{m \geq 1}$, denoted by $\Cdim(H)$, to be the largest $m$ such that $G_m(H)$ has a clique of size $2^m$. They proved that $\Cdim(H)$ is closely related to $\ell = \Ldim(H)$, namely,

$$\ell \leq \Cdim(H) \leq O(\ell \log{\ell}).$$

Theorem \ref{thm:reconstruction-main} implies a stronger graph-theoretic characterization of $\Ldim(H)$ which captures its exact value. We explicitly provide such a characterization in the standard contradiction graphs, but note that the characterization only relies on example sets $\mathcal{V}_e$ which are detectable by duplicate-free graphs in infinite domains, therefore it holds in this alternate setting as well. First, we recall the standard definition of Littlestone dimension.

\begin{definition}[Littlestone dimension~\cite{Littlestone1988}]\label{def:littlestone-dimension}
A \emph{mistake tree of depth $d$} is a complete rooted binary tree of depth $d$ whose internal nodes are labeled by points of $X$ and whose two outgoing edges at every internal node are labeled $0$ and $1$, one of each.  A root-to-leaf path determines a length-$d$ labeled sequence by recording, at each internal node, the point labeling that node together with the label of the edge taken.  The tree is \emph{shattered by $H$} if every root-to-leaf sequence is $H$-realizable.  The \emph{Littlestone dimension} of $H$ is
\[
\Ldim(H)
=
\sup\bigl\{d\in\mathbb N:H\text{ shatters a mistake tree of depth }d\bigr\}.
\]
We use the convention that the supremum of the empty set is $0$, and that the supremum is $\infty$ if $H$ shatters mistake trees of arbitrarily large finite depth.
\end{definition}

We now present the precise graph theoretic structure encoding the Littlestone dimension.

\begin{definition}\label{def:abstract-mistake-tree}
An \emph{abstract mistake tree of depth $d$} in
$G_d(H)$ consists of an ordered complementary pair
\[
(\mathcal V_s^0,\mathcal V_s^1)
\]
of example sets for every $s\in\bits^{<d}$, together with
vertices
\[
(q_\eta)_{\eta\in\bits^d}\subseteq V(G_d(H)),
\]
such that, for every $s\in\bits^{<d}$, $b\in\bits$, and
$\eta\in\bits^d$,
\[
sb\preceq\eta
\quad\Longrightarrow\quad
q_\eta\in\mathcal V_s^b.
\]
Equivalently, for every $\eta\in\bits^d$, the example sets
selected along the path to $\eta$ have nonempty intersection.
\end{definition}

\begin{theorem}\label{thm:df-abstract-tree}
Let $d\geq 2$.  Then
\[
\Ldim(H)\geq d
\quad\Longleftrightarrow\quad
G_d(H)
\text{ contains an abstract mistake tree of depth }d.
\]
\end{theorem}


The proof is straightforward and is thus deferred to Appendix \ref{apx:proofs}. With this, we get the desired level-by-level characterization.

\begin{corollary}\label{cor:ldim-level-certificate}
For every hypothesis class $H$ and every $d\ge2$, the single graph $G_d(H)$ determines whether $\Ldim(H)\ge d$.
\end{corollary}

\begin{proof}
By Theorem~\ref{thm:df-abstract-tree}, the predicate $\Ldim(H)\ge d$ is equivalent to the existence in $G_d(H)$ of the graph-theoretic configuration from Definition~\ref{def:abstract-mistake-tree}.  This property is preserved by graph isomorphism.
\end{proof}

For $d=1$, the corresponding statement is simply that
$\Ldim(H)\geq 1$ if and only if
$G_1(H)$ contains an edge.

\section{Discussion}\label{sec:discussion}

In this paper, we proved that the realizable samples at every fixed sample length are determined by the sequence $(G_m(H))_{m\geq 1}$ of contradiction graphs. Consequently, every learning-theoretic property of this structure that is invariant under signed relabeling is also determined by the contradiction-graph sequence. 

It is natural to ask whether learning-equivalence implies that two concept classes are isomorphic. That is, can the dependence on $m$ of the permutations $\pi_m$ and maps $\epsilon_m$ be removed? In fact, the dependence on $m$ is necessary.

To illustrate this, let $E_n$ and $F_n$, with $n\geq2$, be finite sets such that $|E_n|=|F_n|=n$, and assume that all the sets $E_n$ and $F_n$ are pairwise disjoint. Furthermore, define the domains 

$$
X=\bigcup_{n\geq 2}E_n, \qquad Y=\left(\bigcup_{n\geq 2}F_n\right)\cup\{\star\}.
$$
    
and the hypothesis classes 

$$H=\left\{h\in\{0,1\}^{X}:h|_{E_n}\neq\mathbf{1}_{E_n} \text{ for all }n\geq 2\right\}$$

\noindent and

$$K=\left\{k\in\{0,1\}^{Y}:k|_{F_n}\neq\mathbf{1}_{F_n} \text{ for all }n\geq 2\right\},$$

\noindent where $\mathbf{1}_A$ denotes the all-one partial assignment on $A$.

Fix $m\geq 1$, and choose a bijection $\pi_m:X\to Y$ that maps $E_n$ onto $F_n$ for every $2\leq n\leq m$. Such a bijection exists because, after these finitely many blocks have been matched, the two remaining sets are countably infinite. For any length-$m$ labeled sequence $S$,

$$ S\text{ is $H$-realizable} \quad\Longleftrightarrow\quad \pi_m(S)\text{ is $K$-realizable}.$$

Indeed, aside from inconsistent repeated labels, a length-$m$ sequence is not $H$-realizable precisely when it contains the all-one labeling of $E_n$ for some $2\leq n\leq m$. The corresponding statement holds for $K$ with $F_n$ in place of $E_n$, and $\pi_m$ carries these obstructions bijectively to one another. Moreover, applying $\pi_m$ entrywise induces a graph isomorphism

$$G_m(H)\cong G_m(K).$$
 
On the other hand, suppose that a single signed relabeling $(\pi,\epsilon)$ preserved realizability at every finite level. The inclusion-minimal non-$H$-realizable finite partial assignments are precisely

$$\{\mathbf{1}_{E_n}:n\geq 2\},$$

\noindent while the inclusion-minimal non-$K$-realizable finite partial assignments are precisely

$$\{\mathbf{1}_{F_n}:n\geq 2\}.$$

A signed relabeling preserves inclusion-minimality and support size. Because there is exactly one such obstruction of each size $n\geq 2$, it follows that

$$\pi(E_n)=F_n\qquad\text{for every }n\geq 2.$$
    
Consequently,

$$\pi(X)=\bigcup_{n\geq 2}F_n=Y\setminus\{\star\},$$
    
\noindent contradicting the assumption that $\pi:X\to Y$ is a bijection. Thus the levelwise signed relabelings cannot, in general, be chosen consistently. 


It is therefore interesting to ask which learning-theoretic properties whose definitions are not confined to realizable samples of a single fixed length are invariant under learning-equivalence, and whether, for such properties, there are natural graph-theoretic structures in the contradiction graphs that capture them exactly or approximately.

\section*{Acknowledgments}
This research was supported in part by NSF grant ECCS-2217023.
The authors used \texttt{ChatGPT 5.4/5.5 Pro} to help explore the structure of contradiction graphs and developed some of the proofs (in particular the duplicate-free characterization) with the help of \texttt{ChatGPT 5.6 Sol}. The authors verified the correctness and originality of all content.

\bibliographystyle{alpha}
\bibliography{main}

\appendix

\section{Subcube partitions, DNF tautologies, and biclique covers}\label{apx:misc}

We describe several equivalent ways to view the combinatorial structure of a clique in the contradiction graph. Namely, a clique naturally gives a packing of pairwise disjoint subcubes, a disjoint DNF, and a biclique cover of the complete graph. When the clique has the maximal size $2^m$, the packing covers the entire Boolean cube, and hence becomes a homogeneous subcube partition/disjoint DNF tautology.

These correspondences allow results about subcube partitions, DNF tautologies, bounded-occurrence SAT, and biclique covers to be applied directly to cliques in contradiction graphs. For instance, we observe sufficient conditions for a clique to organize perfectly into a shattered mistake tree. Moreover, the natural biclique cover of a clique in the contradiction graph recovers a key lemma of Alon et al.~\cite{AlonMoranScheflerYehudayoff2024} with a simple proof.

Conversely, we also observe that cliques in the contradiction graph lend a useful perspective towards solving problems in these other equivalent settings. For instance, we show that the natural contradiction (conflicting) structure of DNF tautologies improves a lower bound for variable occurrences of Sloan et al.~\cite{SSR08}.

\subsection{Cliques and subcube partitions}\label{apx:subcube-partitions}

Let $H\subseteq\bits^X$, let $m\geq 1$, and let
\[
C\subseteq V(G_m(H))
\]
be a finite family of realizable samples. Since every $S\in C$ is realizable, it induces a well-defined partial assignment
\[
p_S:\supp(S)\longrightarrow\bits,
\qquad
p_S(x)=b\quad\Longleftrightarrow\quad (x,b)\in S.
\]

Let
\[
Y_C:=\bigcup_{S\in C}\supp(S).
\]
The sample $S$ determines the boolean subcube of $\bits^{Y_C}$,
\[
Q_S
:=
\left\{
a\in\bits^{Y_C}:
a(x)=p_S(x)\text{ for every }x\in\supp(S)
\right\}.
\]
Thus $Q_S$ has codimension $|\supp(S)|$ in $\bits^{Y_C}$ and relative volume
\[
\mu(Q_S)=2^{-|\supp(S)|}.
\]

\begin{proposition}\label{prop:cliques-subcube-packings}
For distinct $S,T\in C$,
\[
S\sim T\text{ in }G_m(H)
\qquad\Longleftrightarrow\qquad
Q_S\cap Q_T=\emptyset.
\]
Consequently, $C$ is a clique if and only if
\[
\mathcal Q(C):=\{Q_S\}_{S \in C}
\]
is a family of pairwise disjoint subcubes of $\bits^{Y_C}$.
\end{proposition}

\begin{proof}
The subcubes $Q_S$ and $Q_T$ intersect if and only if the partial assignments $p_S$ and $p_T$ are compatible. They are incompatible precisely when there is some $x\in\supp(S)\cap\supp(T)$ for which
\[
p_S(x)\neq p_T(x).
\]
This is exactly the condition that $S$ and $T$ contradict at $x$, and hence are adjacent in $G_m(H)$.
\end{proof}

Thus an arbitrary clique corresponds to a \emph{subcube packing}. This packing is strengthened to a partition when the clique has maximal size.

\begin{definition}
A clique $C\subseteq V(G_m(H))$ is \emph{extremal} if
\(
|C|=2^m.
\)
\end{definition}

\begin{proposition}\label{prop:extremal-homogeneous-partition}
Let $C\subseteq V(G_m(H))$ be a clique. Then
\[
\sum_{S\in C}2^{-|\supp(S)|}\leq 1.
\]
In particular, $|C|\leq 2^m$. If $|C|=2^m$, then:
\begin{enumerate}[label=(\roman*)]
\item $|\supp(S)|=m$ for every $S\in C$, that is, no sample in $C$ contains a repeated point;
\item the subcubes $(Q_S)_{S\in C}$ partition $\bits^{Y_C}$;
\item this partition is homogeneous of codimension $m$.
\end{enumerate}
\end{proposition}

\begin{proof}
By Proposition~\ref{prop:cliques-subcube-packings}, the subcubes $(Q_S)_{S\in C}$ are pairwise disjoint. Therefore
\[
\sum_{S\in C}2^{-|\supp(S)|}
=
\sum_{S\in C}\mu(Q_S)
=
\mu\left(\bigcup_{S\in C}Q_S\right)
\leq 1.
\]
Since $|\supp(S)|\leq m$ for every length-$m$ sample,
\[
1
\geq
\sum_{S\in C}2^{-|\supp(S)|}
\geq
|C|2^{-m}.
\]
It follows that $|C|\leq 2^m$.

Now suppose $|C|=2^m$. Equality must hold throughout the preceding display. In particular,
\[
2^{-|\supp(S)|}=2^{-m}
\]
for every $S\in C$, and hence $|\supp(S)|=m$. Moreover,
\[
\mu\left(\bigcup_{S\in C}Q_S\right)=1.
\]
Because $Y_C$ is finite, every point of $\bits^{Y_C}$ has positive measure. Thus the union is all of $\bits^{Y_C}$, and the $Q_S$ form a homogeneous codimension-$m$ subcube partition.
\end{proof}

The converse construction is equally direct. Thus homogeneous codimension-$k$ subcube partitions and extremal cliques in order-$k$ contradiction graphs are equivalent, up to the choice of a realizing concept class.

Consequently, the literature on homogeneous and irreducible subcube partitions can be viewed as studying structured extremal cliques in contradiction graphs~\cite{FHK+22,SSR08, FKW02}. Reducibility of a subcube partition also has a natural clique interpretation: a proper subfamily of vertices is reducible precisely when the union of its associated cells is itself a subcube, so that the subfamily can be replaced by a single shorter partial assignment.

\subsection{DNF tautologies}\label{apx:dnf-cnf}

Introduce one Boolean variable $z_x$ for every $x\in Y_C$, and write
\[
z_x^{\,b}
:=
\begin{cases}
z_x,&b=1,\\
\neg z_x,&b=0.
\end{cases}
\]
The partial assignment associated with $S$ determines the term
\[
T_S
:=
\bigwedge_{x\in\supp(S)}z_x^{\,p_S(x)}.
\]
The satisfying assignments of $T_S$ are exactly the points of the subcube $Q_S$.

Two terms are said to \emph{conflict} if one contains $z_x$ and the other contains $\neg z_x$ for some variable $z_x$. Equivalently, their satisfying subcubes are disjoint. A DNF whose terms conflict pairwise is called a \emph{disjoint DNF}.

\begin{proposition}\label{prop:partition-dnf}
Let $\mathcal F=\{p_i:D_i\to\bits:i\in I\}$ be a finite family of partial assignments, and define
\[
\Phi_{\mathcal F}
:=
\bigvee_{i\in I}
\left(
\bigwedge_{x\in D_i}z_x^{\,p_i(x)}
\right).
\]
Then:
\begin{enumerate}[label=(\roman*)]
\item the corresponding subcubes are pairwise disjoint if and only if the terms of $\Phi_{\mathcal F}$ conflict pairwise;
\item the subcubes cover the Boolean cube if and only if $\Phi_{\mathcal F}$ is a tautology;
\item $\mathcal F$ is a subcube partition if and only if $\Phi_{\mathcal F}$ is a disjoint DNF tautology;
\item $\mathcal F$ is a homogeneous codimension-$k$ subcube partition if and only if $\Phi_{\mathcal F}$ is a disjoint $k$-DNF tautology (i.e. each term consists of exactly $k$ literals)
\end{enumerate}
\end{proposition}

In particular, an arbitrary clique $C\subseteq V(G_m(H))$ gives the disjoint DNF
\[
\Phi_C:=\bigvee_{S\in C}T_S.
\]
That said, it need not be a tautology. If $C$ is extremal, however, Proposition~\ref{prop:extremal-homogeneous-partition} shows that
\[
\Phi_C
=
\bigvee_{S\in C}
\left(
\bigwedge_{x\in\supp(S)}z_x^{\,p_S(x)}
\right)
\]
is a disjoint $m$-DNF tautology with exactly $2^m$ terms.

This clique--DNF correspondence implies a somewhat surprising property of extremal cliques. The main theorem of Sloan, Sz\"or\'enyi, and Tur\'an~\cite{SSR08} states that a disjoint DNF tautology in which every two terms conflict in exactly one variable is generated by a nonrepeating decision tree. Sz\"or\'enyi~\cite{S08} strengthened this from conflict bound one to conflict bound two. That is, if every two terms conflict in one or two variables, then the tautology is still generated by a decision tree.

\begin{proposition}[cf.~\cite{SSR08, S08}]\label{prop:conflict-two-tree}
Let $C\subseteq V(G_m(H))$ be an extremal clique. If every two datasets of $C$ contradict on at most $2$ examples, $C$ is exactly the family of root-to-leaf samples of a complete binary mistake tree of depth $m$ shattered by $H$. In particular,
\[
\Ldim(H)\geq m.
\]
\end{proposition}

Conversely, this correspondence is also interesting as it lends a new perspective towards studying DNF tautologies. For example, let

\[
\Phi=T_1\vee\cdots\vee T_r
\]
be a disjoint DNF tautology over variables $z_1,\ldots,z_n$, all of which are \emph{active}, meaning that each occurs in at least one term. For each term, put
\[
w_i:=2^{-|T_i|}.
\]
Because the satisfying subcubes of the terms partition the Boolean cube,
\[
\sum_{i=1}^r w_i=1.
\]
For $j\in[n]$, define
\[
p_j^+
:=
\sum_{i:z_j\in T_i}w_i,
\qquad
p_j^-
:=
\sum_{i:\neg z_j\in T_i}w_i,
\qquad
v_j:=p_j^++p_j^-.
\]
Thus $v_j$ is the total volume of terms in which $z_j$ occurs in either polarity. Write
\[
\alpha(\Phi):=\max_{j\in[n]}v_j
\]
and define
\[
\alpha_n
:=
\inf\left\{
\alpha(\Phi):
\Phi\text{ is a disjoint DNF tautology on exactly $n$ active variables}
\right\}.
\]

It is known that (cf. \cite{SSR08, SS00}),

$$\frac{\log{n} - \log{\log{n}}}{n} \leq \alpha_n \leq O(n^{-1/5})$$

However, by exploiting the conflicting structure of DNF tautologies, we can easily improve the lower bound to $\alpha_n \geq \Omega(n^{-1/2})$, as we prove next.

\begin{lemma}\label{lem:dnf-volume-balance}
For every $j\in[n]$,
\[
p_j^+=p_j^-=\frac{v_j}{2}.
\]
\end{lemma}

\begin{proof}
The halfcube $z_j=1$ has measure $1/2$. Terms containing $z_j$ contribute their entire volume to this halfcube, terms containing $\neg z_j$ contribute none of their volume, and terms omitting $z_j$ contribute half of their volume. Therefore
\[
p_j^+ + \frac12(1-v_j)=\frac12,
\]
which gives $p_j^+=v_j/2$. Applying the same argument to the halfcube $z_j=0$ gives $p_j^-=v_j/2$.
\end{proof}

\begin{theorem}\label{thm:alpha-lower-bound}
For every $n\geq1$,

$$\alpha_n \geq \frac{\sqrt{2n+1}-1}{n}.$$
\end{theorem}

\begin{proof}
Fix a disjoint DNF tautology $\Phi=T_1\vee\cdots\vee T_r$ on $n$ active variables, and write $\alpha=\alpha(\Phi)$. Since the terms are disjoint, every ordered pair of distinct terms conflicts in at least one variable. Consequently,

\begin{align*}
1-\sum_{i=1}^r w_i^2
&=
\sum_{i\neq k}w_iw_k\\
&\leq
\sum_{j=1}^n
\left(p_j^+p_j^-+p_j^-p_j^+\right)\\
&=
\frac12\sum_{j=1}^n v_j^2\\
&\leq
\frac{n\alpha^2}{2},
\end{align*}
where the equality in the third line uses Lemma~\ref{lem:dnf-volume-balance}.

Because $n\geq1$ and every variable is active, no term is empty. For each $i$, choose a literal of $T_i$, say on variable $z_j$. Then $w_i$ is one of the summands defining $v_j$, so $w_i\leq v_j\leq\alpha$. Hence

$$\sum_{i=1}^r w_i^2 \leq \alpha\sum_{i=1}^r w_i = \alpha.$$

Combining the two estimates yields

$$1-\alpha\leq\frac{n\alpha^2}{2},$$

or equivalently

$$n\alpha^2+2\alpha-2\geq0.$$

Since $\alpha\geq0$, solving this quadratic inequality gives

$$\alpha\geq\frac{\sqrt{2n+1}-1}{n}$$
\end{proof}

\subsection{A biclique cover of $C$}\label{apx:biclique-cover}

Let $C\subseteq V(G_m(H))$ be a clique, and put $q:=|C|$. For $x\in X$ and $b\in\bits$, let
\[
C_{x,b}
:=
\{S\in C:(x,b)\in S\}.
\]
The two sets $C_{x,0}$ and $C_{x,1}$ are disjoint. They determine a biclique
\[
B_x:=K_{C_{x,0},C_{x,1}}
\]
in the complete graph whose vertex set is $C$.

\begin{proposition}\label{prop:natural-biclique-cover}
The family
\[
\mathcal B_C
:=
(B_x:x\in X,\ C_{x,0}\neq\emptyset,\ C_{x,1}\neq\emptyset)
\]
is a biclique cover of $C$. Moreover,
\begin{enumerate}[label=(\roman*)]
\item the multiplicity with which an edge $\{S,T\}$ is covered is exactly
\[
\bigl|\{x\in\supp(S)\cap\supp(T):p_S(x)\neq p_T(x)\}\bigr|,
\]
the number of coordinates on which $S$ and $T$ contradict;
\item every edge is covered at least once and at most $m$ times;
\item every vertex $S\in C$ belongs to at most $m$ bicliques.
\end{enumerate}
Thus $\mathcal B_C$ is simultaneously a biclique cover of order at most $m$ and an $m$-local biclique cover.
\end{proposition}

\begin{proof}
Let $S,T\in C$ be distinct. Since $C$ is a clique, there is some $x\in X$ such that, after possibly interchanging $S$ and $T$,
\[
(x,0)\in S,
\qquad
(x,1)\in T.
\]
Hence $S\in C_{x,0}$ and $T\in C_{x,1}$, so the edge $\{S,T\}$ belongs to $B_x$. Therefore the bicliques cover every edge of $C$.

The edge $\{S,T\}$ belongs to $B_x$ precisely when the two samples assign opposite labels to $x$, proving the multiplicity statement. There are at most $m$ such coordinates because each sample has support of size at most $m$. Finally, $S$ belongs to $B_x$ only if $x\in\supp(S)$, so $S$ belongs to at most $|\supp(S)|\leq m$ bicliques.
\end{proof}

This biclique cover of $C$ provides a simple proof for a key lemma from \cite{AlonMoranScheflerYehudayoff2024} that every clique contains an example appearing in $\Omega(|C|/m)$ datasets. They observe that iterating this lemma constructs a deep, shattered tree whose root-to-leaf datasets are sub-datasets of vertices in $C$.

\begin{proposition}[cf.\ \cite{AlonMoranScheflerYehudayoff2024}]\label{prop:biclique-splitter}
Let $C\subseteq V(G_m(H))$ be a clique of size $q\geq2$. Then some $x\in X$ satisfies
\[
|C_{x,0}|,\ |C_{x,1}|\geq\frac{q-1}{2m}.
\]
\end{proposition}

\begin{proof}
Put $a_x=|C_{x,0}|$, $b_x=|C_{x,1}|$, and $\tau=\max_x\min\{a_x,b_x\}$. Since $\mathcal B_C$ covers $C$ and each sample has support size at most $m$,
\[
\underbrace{\binom q2}_{|E(C)|}\leq\overbrace{\sum_x a_xb_x}^{|E(\mathcal{B}_C)|}
\leq\tau\sum_x(a_x+b_x)
=\tau\sum_{S\in C}|\supp(S)|
\leq\tau qm.
\]

Hence, dividing by $qm$, we obtain $\tau\geq(q-1)/(2m)$.
\end{proof}

\section{Omitted proofs}\label{apx:proofs}

\subsubsection*{Proof of Lemma \ref{lem:minimal-one-moving}} 
\begin{proof}
Suppose first that $S$ contains a unique interesting point, and let $e=(x,b)$ be the corresponding signed example. Every other point appearing in $S$ is boring, and hence
\[
\Nbr(S)=\Nbr(e^m).
\]
Assume that some vertex $T$ satisfies
\[
\emptyset\subsetneq\Nbr(T)\subsetneq\Nbr(S).
\]
Since $\Nbr(T)\neq\emptyset$, the sample $T$ contains an interesting signed example $f=(z,c)$. The repeated sample $\bar f^m$ is realizable and adjacent to $T$, so it is adjacent to $S$. Thus $S$ contains $f$, and uniqueness gives $f=e$. Since $T$ contains $e$, every vertex adjacent to $e^m$ is adjacent to $T$. Therefore
\[
\Nbr(S)=\Nbr(e^m)\subseteq\Nbr(T),
\]
a contradiction. Hence $S$ is minimal.

Conversely, let $S$ be minimal. Since $\Nbr(S)\neq\emptyset$, the sample $S$ contains at least one interesting point. If it contained two distinct interesting signed examples $e=(x,b)$ and $f=(z,c)$, then
\[
\emptyset\subsetneq\Nbr(e^m)\subseteq\Nbr(S).
\]
The last containment is indeed proper because $\bar f^m$ is adjacent to $S$ but not to $e^m$. This contradicts minimality. Thus $S$ contains exactly one interesting point.
\end{proof}

\subsubsection*{Proof of Theorem \ref{thm:df-abstract-tree}} 
Suppose first that $H$ shatters a mistake tree of depth $d$. Index its internal nodes by $s\in\bits^{<d}$, and let $x_s$ be the point labeling node $s$.  Set
\[
\mathcal V_s^b=\mathcal V_{(x_s,b)}.
\]
Both labels of $x_s$ occur on realizable branches, so $x_s$ is interesting. The characterization of the ordinary example sets above shows that $(\mathcal V_s^0,\mathcal V_s^1)$ is an ordered complementary pair.

No point can appear twice on a root-to-leaf path in a shattered mistake tree. Indeed, if $x$ appeared at an ancestor and again at a descendant, one of the two branches at the descendant would assign $x$ the opposite of its earlier label, producing an unrealizable path. Hence every root-to-leaf path gives a duplicate-free labeled sample of size $d$. Let $q_\eta$ be the sample on the path indexed by $\eta$.  It is a vertex of $G_d(H)$, and whenever $sb\preceq\eta$, it contains $(x_s,b)$.  Therefore
\[
q_\eta\in\mathcal V_s^b,
\]
so these objects form an abstract mistake tree.

Conversely, suppose $G_d(H)$ contains an abstract mistake tree. By the characterization of the ordinary example sets above, for each internal node $s$ there is an interesting point $x_s$ and opposite bits $c_s^0,c_s^1$ such that
\[
\mathcal V_s^b=\mathcal V_{(x_s,c_s^b)}.
\]
Label node $s$ by $x_s$, and label the edge from $s$ to $sb$ by $c_s^b$.

Fix $\eta\in\bits^d$.  If $sb\preceq\eta$, then
\[
q_\eta\in\mathcal V_s^b
=
\mathcal V_{(x_s,c_s^b)}.
\]
Thus $q_\eta$ contains every signed example prescribed along the root-to-leaf path indexed by $\eta$. Since $q_\eta$ is $H$-realizable, the entire path is $H$-realizable. This holds for every leaf, so the constructed mistake tree is shattered by $H$.
\qed

\end{document}